\documentclass{article}

\usepackage{microtype}
\usepackage{graphicx}
\usepackage{booktabs} 

\usepackage{hyperref}

\usepackage[preprint]{icml2026}

\usepackage{amsmath}
\usepackage{amssymb}
\usepackage{mathtools}
\usepackage{amsthm}
\usepackage{mathrsfs}

\usepackage{multirow}
\usepackage{subfigure}
\usepackage[normalem]{ulem}
\usepackage{array}
\useunder{\uline}{\ul}{}

\usepackage[capitalize,noabbrev]{cleveref}

\theoremstyle{plain}
\newtheorem{theorem}{Theorem}[section]

\newtheorem{lemma}[theorem]{Lemma}

\theoremstyle{definition}

\theoremstyle{remark}

\usepackage[textsize=tiny]{todonotes}

\icmltitlerunning{Improving Graph Few-shot Learning with Hyperbolic Space and Denoising Diffusion}

\begin{document}

\twocolumn[
  \icmltitle{Improving Graph Few-shot Learning with Hyperbolic \\ Space and Denoising Diffusion}



  \icmlsetsymbol{equal}{*}

  \begin{icmlauthorlist}
    \icmlauthor{Yonghao Liu}{jlu}
    \icmlauthor{Jialu Sun}{jlu}
    \icmlauthor{Wei Pang}{hwu}
    \icmlauthor{Fausto Giunchiglia}{ut}
    \icmlauthor{Ximing Li}{jlu}
    \icmlauthor{Xiaoyue Feng}{jlu}
    \icmlauthor{Renchu Guan}{jlu}
  \end{icmlauthorlist}

  \icmlaffiliation{jlu}{Key Laboratory of Symbolic Computation and Knowledge Engineering of the Ministry of Education, College of Computer Science and Technology, Jilin University}
  \icmlaffiliation{hwu}{School of Mathematical and Computer Sciences, Heriot-Watt University}
  \icmlaffiliation{ut}{Department of Information Engineering and Computer Science, University of Trento}

  \icmlcorrespondingauthor{Xiaoyue Feng}{fengxy@jlu.edu.cn}
  \icmlcorrespondingauthor{Renchu Guan}{guanrenchu@jlu.edu.cn}

  \icmlkeywords{Machine Learning, ICML}

  \vskip 0.3in
]



\printAffiliationsAndNotice{}  

\begin{abstract}
Graph few-shot learning, which focuses on effectively learning from only a small number of labeled nodes to quickly adapt to new tasks, has garnered significant research attention. Despite recent advances in graph few-shot learning that have demonstrated promising performance, existing methods still suffer from several key limitations. First, during the meta-training phase, these methods typically perform node representation learning in Euclidean space, which often fails to capture the inherently hierarchical structure existing in real-world graph data.
Second, during the meta-testing phase, they usually fit an empirical target distribution derived from only a few support samples, even when this distribution significantly deviates from the true underlying distribution. To address these issues, we propose \textbf{IMPRESS}, a novel framework that \textbf{IM}proves gra\textbf{P}h few-shot learning with hype\textbf{R}bolic spac\textbf{E} and denoi\textbf{S}ing diffu\textbf{S}ion. Specifically, our model learns node representations in a hyperbolic space and enriches the support distribution through denoising diffusion mechanisms. Theoretically, IMPRESS achieves a tighter generalization bound. Empirically, IMPRESS consistently outperforms competitive baselines across multiple benchmark datasets. 
\end{abstract}

\section{Introduction}
Graph Neural Networks (GNNs) have become the \textit{de facto} standard for modeling graph-structured data due to their unique advantages in capturing complex relational structures \cite{wu2020comprehensive, liu2025high}. However, most existing GNN models heavily rely on a large amount of labeled data to fully unleash their potential \cite{huang2020graph, liu2022few}. This requirement is often impractical in real-world scenarios, where obtaining high-quality annotations can be prohibitively expensive or even infeasible \cite{liu2025enhancing, liu2025dual}. For example, in the field of bioinformatics, accurately annotating unknown genes requires deep expertise in molecular biology, posing a significant challenge even for seasoned researchers \cite{hu2020open}. Given these challenges, graph few-shot learning (FSL) has emerged as a promising solution. By requiring only a few labeled nodes, graph FSL enables models to rapidly adapt to new tasks, making it an attractive research direction \cite{zhang2025few, liu2025graph}. Most existing graph FSL models follow a two-stage paradigm: during the meta-training phase, they perform node representation learning and capture generalizable knowledge across multiple tasks; during the meta-testing phase, they leverage this prior knowledge to adapt quickly to target tasks using only limited labeled data \cite{zhou2019meta, kim2023task}. Despite the remarkable progress made in recent years, current graph FSL methods still face several notable limitations that hinder their performance.

First, current graph FSL models typically perform node representation learning in Euclidean space during the meta-training phase. While these operations are more straightforward in Euclidean space---benefiting from well-established operators and computational efficiency---real-world graph-structured data often exhibit hierarchical or tree-like distributions \cite{peng2021hyperbolic, yang2022hyperbolic}. 
For example, as shown in Fig. \ref{example} (a), a company’s organizational structure can be modeled as a graph, with employees being represented as nodes and departments or teams as edges. The hierarchy from top management (\textit{e.g.}, CEO) to mid-level managers (\textit{e.g.}, directors) and staff forms a hierarchical graph structure that reflects decision-making processes and task delegation. Moreover, we visualize the hierarchical structure of academic research topics in Fig. \ref{example} (b). Specifically, the research area of artificial intelligence typically encompasses subfields such as natural language processing and computer vision, which in turn contain more specialized topics such as text generation and image classification.
Under such hierarchical organization, Euclidean embeddings often suffer from significant distortion, making it difficult to faithfully capture the underlying structural relationships.

\begin{figure}
    \centering
    \subfigure[Company Orinization]{\includegraphics[width=0.2\textwidth]{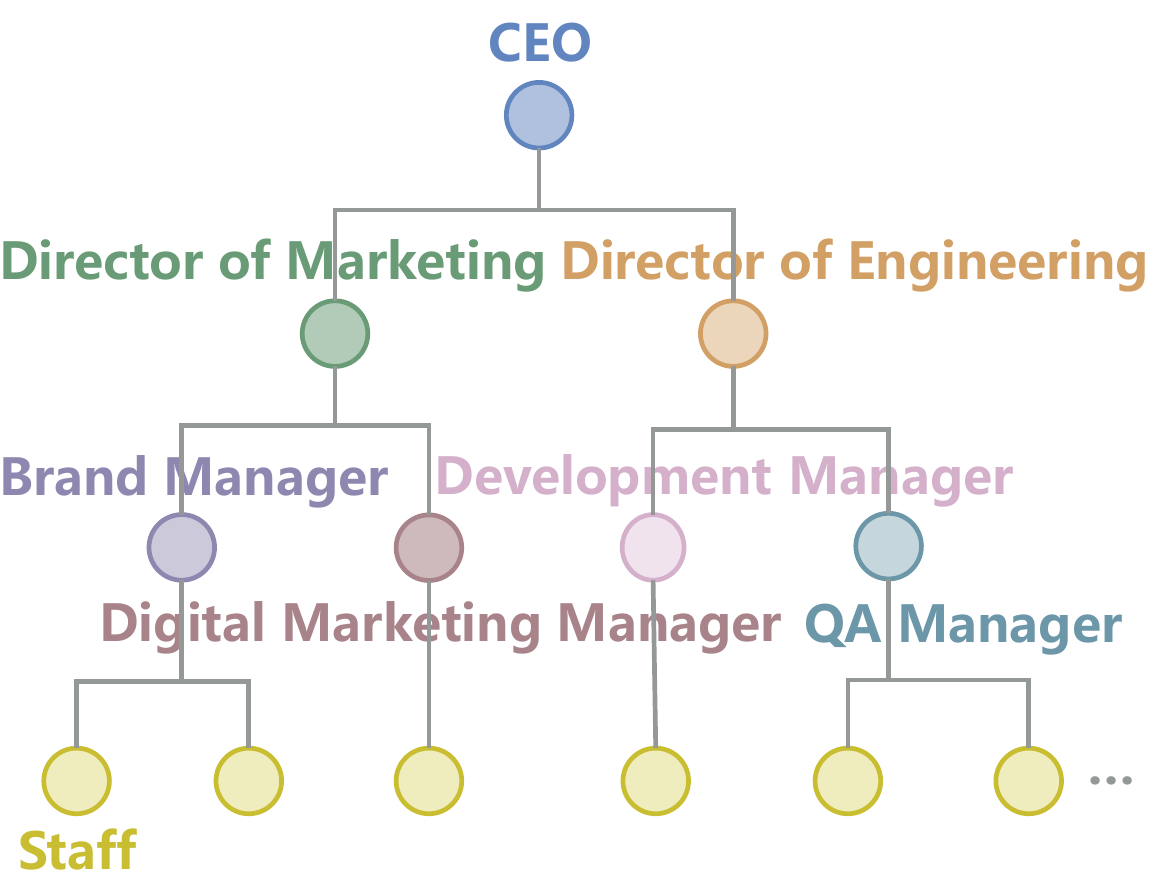}}
    \subfigure[Research Topic]{\includegraphics[width=0.2\textwidth]{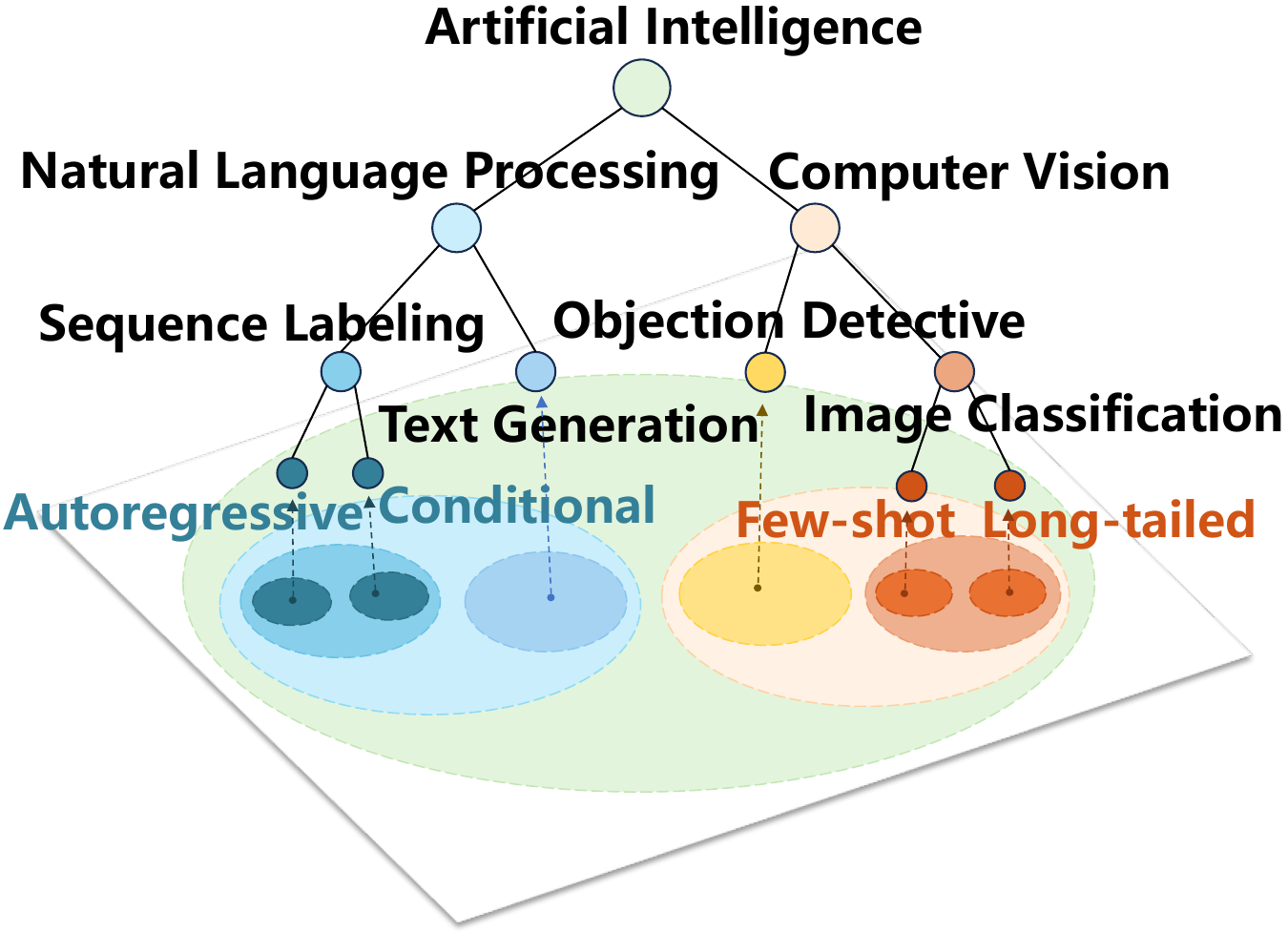}}
    \caption{Examples of hierarchical structures in graph data.}
    \label{example}
\end{figure}

Second, in the meta-testing phase, existing graph FSL models typically perform parameter fine-tuning using only a small set of labeled support nodes. These models implicitly assume that the distribution of the support set can serve as a reliable proxy for the true class distribution, and that training a classifier directly on these samples will yield accurate decision boundaries. 
However, this assumption introduces two critical issues. \textit{On the one hand}, the support set often fails to capture the true data distribution due to its extremely limited size, leading to biased and unrepresentative distributions. To support our claim, we visualize the distributional discrepancy between the support set and the query set in a randomly sampled task from the Cora and CiteSeer datasets \cite{yang2016revisiting} using a histogram-based density plot. As shown in Fig. \ref{vis}, the support set clearly exhibits a biased distribution, which deviates from that of the query set.
\textit{On the other hand}, training the downstream classifier directly on such sparse data is highly susceptible to overfitting, which significantly impairs the model’s generalization ability.

\begin{figure}
    \centering
    \subfigure[Cora]{\includegraphics[width=0.2\textwidth]{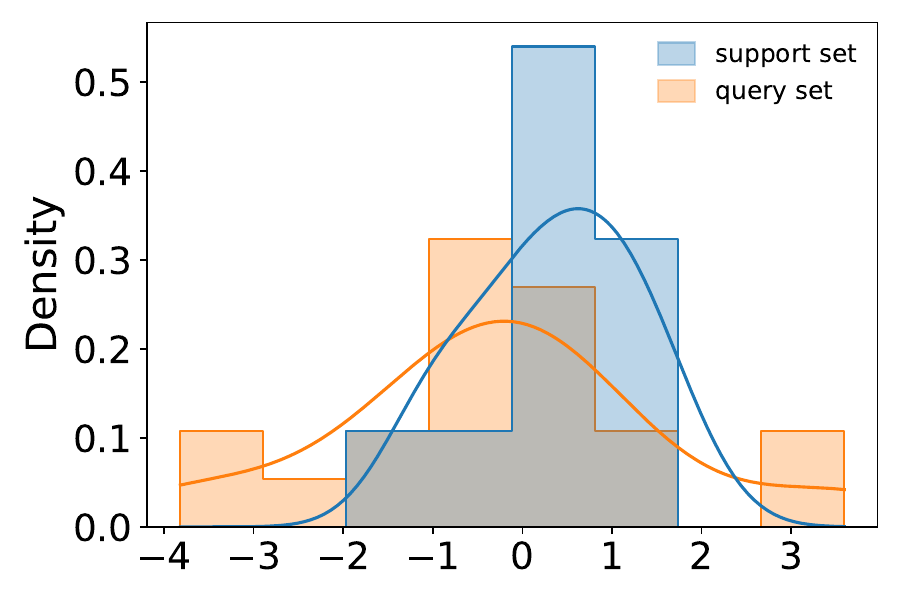}}
    \subfigure[CiteSeer]{\includegraphics[width=0.2\textwidth]{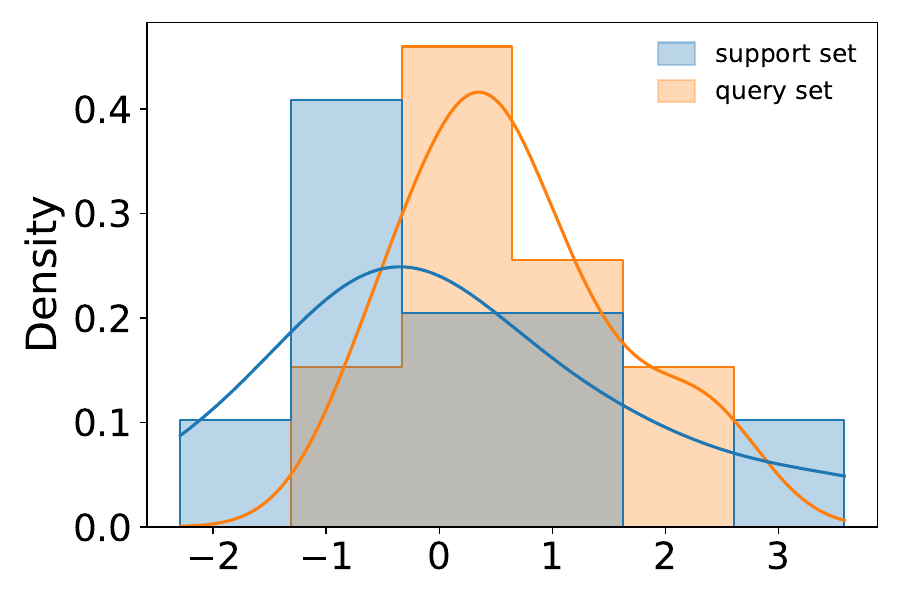}}
    \caption{Data Distribution of randomly selected support set and query set from the Cora and CiteSeer datasets.}
    \label{vis}
\end{figure}

To address the aforementioned issues, we propose a novel framework named \textbf{IMPRESS} that \textbf{IM}poves gra\textbf{P}h few-shot learning with hype\textbf{R}bolic spac\textbf{E} and denoi\textbf{S}ing diffu\textbf{S}ion.
Specifically, during the meta-training phase, we first train a hyperbolic variational graph autoencoder to explicitly capture the hierarchical structure inherent in graph data. In the meta-testing phase, inspired by the powerful generative capabilities of latent diffusion models \cite{rombach2022high}, we aim to leverage them to generate a large number of support node embeddings in the latent space, thereby enabling a more accurate approximation of the class decision boundaries.
However, a major challenge arises: due to the limited number of labeled support samples in few-shot scenarios, it is impractical to train a diffusion model solely on the support set. To address this, we leverage the abundant data available during meta-training to train the diffusion model.
Nevertheless, this introduces another challenge---since the class labels in meta-training and meta-testing are disjoint, the diffusion model trained on base classes may fail to generate samples that align with the novel class distributions.
To overcome this issue, we incorporate class prototype information as a conditioning signal during the diffusion training process. This guidance helps the model generate embeddings that better align with the target class, under the supervision of its corresponding prototype.
Theoretically, our method achieves a tighter generalization bound compared to prior approaches. Empirically, it consistently outperforms existing models across multiple benchmark datasets. In summary, our contributions are as follows.

(I) We propose a novel framework, namely IMPRESS, which improves graph FSL with hyperbolic space and denoising diffusion.

(II) We theoretically elucidate the underlying mechanisms that contribute to the effectiveness of our method.

(III) We empirically demonstrate that our approach achieves superior performance across multiple evaluation datasets.
\section{Preliminary Study}
\textbf{Experimental Setting.} Given a graph $\mathcal{G} = \{\mathcal{V},\mathcal{E},X,A\}$, where $\mathcal{V} = \{v_1, \cdots, v_n\}$ and $\mathcal{E} = \{e_1, \cdots, e_m\}$ represent the sets of nodes and edges, respectively. $X \in \mathbb{R}^{n \times d}$ denotes the node feature matrix, and $A \in \mathbb{R}^{n \times n}$ is the adjacency matrix. In this work, we focus on the widely studied few-shot node classification problem in graph FSL setting to assess model performance. In conventional \textit{few-shot node classification}, the meta-training task set $\mathcal{D}_{\text{tra}}$ comprises a collection of tasks $\mathcal{T}_{\text{tra}} = \{\mathcal{S}_{\text{tra}}, \mathcal{Q}_{\text{tra}}\}$, where $\mathcal{S}_{\text{tra}} = \{(v^{\text{tra}}_i, y^{\text{tra}}_i)\}_{i=1}^{NK}$ denotes the support set, and $\mathcal{Q}_{\text{tra}} = \{(v^{\text{tra}}_i, y^{\text{tra}}_i)\}_{i=1}^{NQ}$ denotes the corresponding query set. Each label $y^{\text{tra}}_i$ belongs to the base class set $\mathcal{Y}_{\text{tra}}$. Likewise, the meta-testing task $\mathcal{T}_{\text{tes}} = \{\mathcal{S}_{\text{tes}}, \mathcal{Q}_{\text{tes}}\}$ is constructed in the same manner as in meta-training, except that the labels come from a novel class set $\mathcal{Y}_{\text{tes}}$, where $\mathcal{Y}_{\text{tra}} \cap \mathcal{Y}_{\text{tes}} = \emptyset$. When the support set $\mathcal{S}_{\text{tes}}$ includes $N$ classes and each class contains $M$ labeled nodes, the task is referred to as an $N$-way $M$-shot task.
Generally, graph FSL models are trained over $\mathcal{D}_{\text{tra}}$, fine-tuned on $\mathcal{S}_{\text{tes}}$, and evaluated on $\mathcal{Q}_{\text{tes}}$. However, in our work, we focus on a more challenging setting---\textit{unsupervised few-shot node classification} \cite{tan2023virtual, jung2024unsupervised}. Specifically, while the meta-testing stage remains the same as traditional few-shot node classification, the meta-training phase is fundamentally different. In this scenario, we discard node labels entirely and instead train on a unified \textit{unlabeled} set $\mathcal{D}_{\text{tra}} = \{v_i\}_{i=1}^{sum}$.

\noindent \textbf{Poincar{\'e} Ball.} It is a widely adopted model of hyperbolic geometry due to its conformal property and numerical stability \cite{peng2021hyperbolic}. Typically, an open $n$-dimensional ball is defined as follows:
\begin{equation}
    \mathcal{B}^n_c = \left\{ x \in \mathbb{R}^n : \|x\|^2 < -1/c \right\},
    \label{eq:1}
\end{equation}
where $c<0$ is negative curvature, and $1/\sqrt{|c|}$ is radius. 

\noindent \textbf{Tangent Space.} For a point $x\in\mathcal{B}_{c}^n$, the tangent space $\mathscr{T}_x\mathcal{B}^n_{c}$ under the negative curvature \(c\) is defined as the vector space consisting of all tangent vectors $v$ at that point, which can be formally defined as follows:
\begin{equation}
    \mathscr{T}_{x}\mathcal{B}^n_c = \{ v \in \mathbb{R}^n: \langle v,x\rangle_{L}=0\},
\end{equation}
where $\langle\cdot,\cdot\rangle_L$ is the Lorentzian inner product. To enable information transformation between the tangent space and the hyperbolic space, a commonly adopted strategy \cite{peng2021hyperbolic} involves the use of the exponential and logarithmic maps to perform bidirectional projections---mapping points $v$ from the tangent space to the hyperbolic manifold and vice versa. These mapping operations are formally defined as follows:
\begin{equation}
\label{eq:2}
\begin{aligned}
    &\exp ^c_x(v)=x \oplus_c \left( \tanh\left( \sqrt{|c|} \frac{\lambda_x^c \|v\|_2}{2} \right) \cdot \frac{v}{\sqrt{|c|} \|v\|_2} \right), \\
    &\log^c_{x}(u) = \frac{2\operatorname{arctanh}\left( \sqrt{|c|} \left\| -x \oplus_c u \right\|_2 \right)}{\sqrt{|c|} \lambda^c_{x}} \cdot  \frac{-x \oplus_c u}{\left\| -x \oplus_c u \right\|_2},
\end{aligned}
\end{equation}
where $x$ is the base point, $\oplus_c$ denotes the M{\"o}bius addition and $\lambda_x^c = 2/(1+c\|x\|_2)$ is the conformal factor. When $x=o$, the tangent space is constructed at the origin of the Poincar{\'e} ball, Eq.\ref{eq:2} can be simplified as follows:
\begin{equation}
\begin{aligned}
    &\exp^c_o(v) = \tanh\left( \sqrt{|c|} \|v\|_2 \right) \cdot \frac{v}{\sqrt{|c|} \|v\|_2}, \\
    &\log^c_{o}(u) = \operatorname{artanh}\left( \sqrt{|c|} \left\| u \right\|_2 \right) \cdot \frac{u}{\sqrt{|c|} \left\|u\right\|_2}.
    \label{eq:3}
\end{aligned}
\end{equation}

To facilitate the understanding, the notations used throughout this research are summarized in \textbf{Appendix} \ref{description_symbols}.
\section{Method}

In this section, we present a detailed description of our proposed model, IMPRESS, which comprises three main components: \textit{hyperbolic node representation learning}, \textit{prototype-guided denoising diffusion}, and the \textit{model prediction} module. 
The overall workflow of IMPRESS is illustrated in Fig. \ref{fig:framework} for better clarify.

\begin{figure*}[htbp]
    \centering{\includegraphics[width=0.72\textwidth]{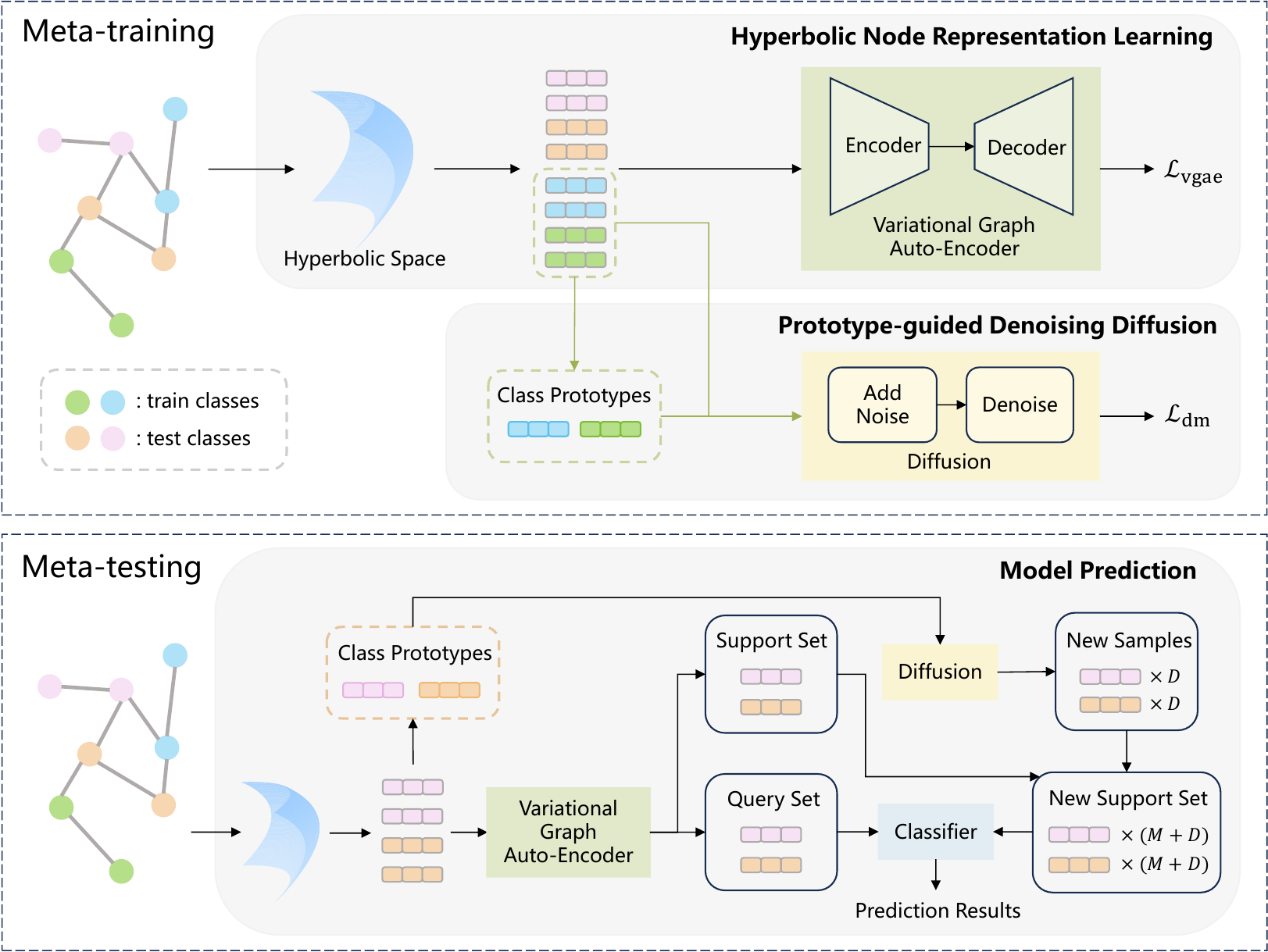}}
    \caption{The overall framework of our model.}
    \label{fig:framework}
\end{figure*}

\subsection{Hyperbolic Node Representation Learning}
The first step in graph FSL typically involves node representation learning, which aims to capture discriminative features of nodes. 
As discussed before, existing graph FSL models operate in Euclidean space, which inevitably limits their ability to capture the hierarchical structures inherent in graphs. Thus, we incorporate hyperbolic operations into the node representation module to explicitly model such hierarchical geometric properties.
Specifically, we first insert zero elements at the 0-th dimension of the node's Euclidean features, and then apply the exponential map to project it from the tangent space to the hyperbolic space, as shown below:
\begin{equation}
    X_{i}^{0,\mathbb H} = \exp_{o}^{c}([0, X_i]),
    \label{eq:hyperbolic}
\end{equation}
where $X_i^{0,\mathbb H}$ is the initial hyperbolic node state. By prepending zero elements to $X$, we ensure that $[0, X_i]$ resides in the tangent space at the origin of the hyperbolic space, thereby aligning it with the underlying hyperbolic geometry \cite{yang2022hicf}.

Next, to avoid relying on any label information during the meta-training phase, we adopt a variational graph autoencoder to learn hyperbolic latent representations of nodes.

\subsubsection{Encoder}
Given the complexity of directly performing feature transformations and nonlinear operations in hyperbolic space \cite{ganea2018hyperbolic}, we instead leverage exponential and logarithmic maps via the tangent space to achieve this objective in an indirect yet effective manner.
Concretely, we project the hyperbolic initial node states to the tangent space via the logarithmic map. Then, we perform graph convolution operations $\Phi(\cdot)$ in the tangent space. The above processes can be formally expressed as follows:
\begin{equation}
\label{encoder}
    \begin{aligned}
        &X^{0,\mathscr{T}}=\log_o^c(X^{0,\mathbb H}), \\
        &X^{t+1,\mathscr{T}}=\Phi(A,X^{t,\mathscr{T}})=\text{ReLU}(\hat{A}X^{t,\mathscr{T}}W^t),
    \end{aligned}
\end{equation}
where $X^{0,\mathscr{T}}$ and $X^{t,\mathscr{T}}$ denote the original and the $t$-th node features in the tangent space, respectively. $\hat{A} = \tilde{D}^{-\frac{1}{2}} \tilde{A} \tilde{D}^{-\frac{1}{2}}$ represents the symmetrically normalized adjacency matrix with self-loops, where $\tilde{A}=A+I$ and $\tilde{D}_{ii} = \sum_{j}\tilde{A}_{ij}$. Moreover, $W^t$ and ReLU$(\cdot)$ denote the layer-specific learnable weight matrix and the nonlinear activation function, respectively. 

After conducting $T$ layers of graph convolution, we obtain the node embeddings denoted as $X^{T,\mathscr{T}}$. Then, we map $X^{T,\mathscr{T}}$ back to the hyperbolic space via the exponential map, as shown below:
\begin{equation}
\label{covert_hyperbolic}
    X^\mathbb H=\exp_o^c(X^{T,\mathscr{T}}).
\end{equation}

We transform the learned hyperbolic node embeddings into latent variable distributions using a specific function, which can be expressed as follows: 
\begin{equation}
\label{vgae_distribution}
    q(Z|X^\mathbb H,A)=\prod_{i=1}^Nq(z_i|X^\mathbb H,A)=\prod_{i=1}^N\mathcal{N}(z_i|\mu_{z_i},\text{diag}(\sigma_{z_i}^2)),
\end{equation}
where $q(Z|X^\mathbb H,A)$ is the joint distribution of latent variable for all nodes. $\mu_{z_i}$ and $\sigma^2_{z_i}$ denote the mean vector of the Gaussian distribution and the associated variance vector. They can be modeled by the following functions:
\begin{equation}
    \label{vgae_matrix}
    \begin{aligned}
        &Z^\prime= \phi(A\log_o^c( X^\mathbb H)\Theta^\prime), \\
        &Z_\mu=AZ^\prime\Theta_\mu, \quad Z_\sigma=AZ^\prime\Theta_\sigma,
    \end{aligned}
\end{equation}
where $Z_\mu$ and $Z_\sigma$ denote the matrix of mean vectors $\mu_{z_i}$ and variance vectors $\sigma^2_{z_i}$. They share the first-layer parameters $\Theta^\prime$. We employ the reparameterization trick to sample latent variable $z_i$ from the latent distribution, which is defined as follows:
\begin{equation}
\label{vgae_sample}
    z_i=\mu_{z_i}+\sigma_{z_i} \odot\epsilon_{z_i},
\end{equation}
where $\epsilon_{z_i}\sim \mathcal{N}(0,1)$ follows the standard normal distribution and $\odot$ is the Hadamard product.

\subsubsection{Decoder}
For the decoder, we implement it by the inner product between latent variables $z_i$ and $z_j$, expressed as follows:
\begin{equation}
\label{decoder}
\begin{aligned}
  &p(A|Z)=\prod_{i=1}^N\prod_{j=1}^Np(A_{ij}|z_i,z_j), \\ &p(A_{ij}=1|z_i,z_j)=\psi(z_i^\top z_j), 
\end{aligned}
\end{equation}
where $\psi(\cdot)$ represents the sigmoid function.

We optimize the variational lower bound by reconstructing the adjacency matrix and encouraging the learned latent distribution to approximate a standard normal distribution as closely as possible:
\begin{equation}
\label{vgae_loss}
\begin{aligned}
    \mathcal{L}_{\text{vgae}}
    &=\mathbb{E}_{q(Z \mid X^H, A)} \left[ \log p \left( A \mid Z \right) \right] \\
    &- \text{KL} \left[ q(Z \mid X^\mathbb H, A) \parallel p(Z) \right],
\end{aligned}
\end{equation}
where KL$(\cdot)$ is the Kullback-Leibler divergence that is adopted to measure the difference between distributions. 

Upon completing the training of the hyperbolic variational graph autoencoder during the meta-training phase, we can obtain latent node representations that effectively capture the hierarchical structure embedded in the graph.

\subsection{Prototype-guided Denoising Diffusion}
Directly training a classifier on the limited support embeddings available during the meta-testing phase may lead to severe overfitting. Therefore, we train a denoising diffusion model to synthesize a large number of auxiliary samples, enriching the original support distribution. Specifically, we first utilize the trained hyperbolic variational graph autoencoder to obtain numerous clean node embeddings $Z_0$ from the meta-training phase. These embeddings are then subjected to a forward diffusion process, as defined below:
\begin{equation}
\begin{aligned}
        &q(Z_{1:K} | Z_0) = \prod_{k=1}^{K} q(Z_k | Z_{k-1}), \\
    &q(Z_k | Z_{k-1}) = \mathcal{N}\left(Z_k; \sqrt{1 - \beta_k} Z_{k-1}, \beta_k I\right),
\end{aligned}
\label{diffusion_forward}
\end{equation}
where $\{\beta\}_{k=0}^K$ is the schedule. Let $\alpha_k=1-\beta_k$ and $\bar{\alpha}_k=\prod_{r=1}^k\alpha_r$, and the forward diffused sample at time step $k$ can be formulated as follows:
\begin{equation}
    Z_k=\sqrt{\bar{\alpha}_k}Z_0+\sqrt{1-\bar{\alpha}_k}\epsilon,\quad \epsilon \sim \mathcal{N}(0, I).
\end{equation}
The reverse process is approximated as a Markov chain, where the mean is learned by the neural networks and the variance is fixed. Starting from pure random noise, the process can be represented as follows:
\begin{equation}
\begin{aligned}
&p_\theta(Z_{0:K})=p_\theta(Z_K)\prod_{k=1}^Kp_\theta(Z_{k-1}|Z_k), \\
&p_\theta(Z_{k-1}|Z_k)=\mathcal{N}(Z_{k-1};\mu_\theta(Z_k, k),\sigma_k^2I).
\end{aligned}
\end{equation}
Here, $\mu_\theta(\cdot)$ denotes an unconditional neural network, implying that the generative process of the diffusion model cannot be directly controlled. Hence, we propose a prototype-guided diffusion model to solve this problem. Since the label information is unavailable during the meta-training phase, we first employ unsupervised clustering to assign pseudo labels for those unlabeled nodes, and subsequently compute the corresponding pseudo prototypes $\mathcal{P}$, defined as follows:
\begin{equation}
\label{prototype}
    \mathcal P=\{\mathcal{P}_i\}_{i=1}^N, \quad \mathcal P_i=\frac{1}{|\Omega_i|} \sum\nolimits_{z_j\in \Omega_i} z_{j},
\end{equation}
 where $\mathcal P_i$ is the prototype, which is computed as the mean embedding of all samples labeled as class $i$. $Y_i$ represents the set of nodes assigned to class $i$. 

Then, we incorporate the pseudo prototype information as a conditioning signal into the denoising network $\mu_\theta(\cdot)$ via a cross-attention mechanism, as formulated below:
\begin{equation}
\label{attention}
\begin{aligned}
&\mathscr{Q} = W_\mathscr{Q} Z_k, \; \mathscr{K}=W_\mathscr{K} \mathcal P, \\
&\mathscr{V}=W_\mathscr{V} \mathcal P, \; \mathscr{S}=\text{softmax}\left(\frac{\mathscr{Q}\mathscr{K}^\top}{\sqrt{d}}\right)\mathscr{V},      
\end{aligned}
\end{equation}
where $W_\mathscr{Q}$, $W_\mathscr{K}$, and $W_\mathscr{V}$ are the trainable parameters. $d$ is the dimension of $\mathscr{Q}$. The node hidden representations $\mathscr{S}$ updated through the cross-attention mechanism are fed into the next block of the denoising model for further refinement.

Finally, we optimize the conditional denoising diffusion model using the following objective function:
\begin{equation}
\label{dm_loss}
    \mathcal{L}_{\text{dm}}= \left\Vert Z - \mu_\theta \left( \sqrt{\bar{\alpha}_k} Z_0 + \sqrt{1 - \bar{\alpha}_k} \epsilon, k, \mathcal{P} \right) \right\Vert^2_2.
\end{equation}
Here, $\mu_\theta(\cdot)$ operates on diffused latent embeddings, the time step, and the conditioning signal, which can be implemented by many neural network architectures.

\subsection{Model Prediction}
In the meta-testing stage, we feed each meta-testing task $\mathcal{T}_\text{tes}$, consisting of a support set $\mathcal{S}_\text{tes}$ and a query set $\mathcal{Q}_\text{tes}$, into the pre-trained variational graph autoencoder to obtain the latent representations of nodes, $Z_\text{spt}$ and $Z_\text{qry}$. Since the support set is labeled, we can compute the class prototypes $P_\text{tes}$ accordingly based on Eq.\ref{prototype}, and obtain the conditioning signal $\mathscr{S}_\text{tes}$ based on Eq.\ref{attention}. Next, we sample $N \times D$ noise vectors $Z_{\text{gen},K}$ from a standard normal distribution and iteratively denoise them using the learned prototype-guided denoising diffusion model, \textit{i.e.}, $Z_{\text{gen},k-1}=\mu_\theta(Z_{\text{gen},k},k,\mathscr{S}_\text{tes})$. After $K$ steps of denoising, we generate $D$ additional clear embeddings $Z_{\text{gen},0}=\mu_\theta(Z_{\text{gen},1},1,\mathscr{S}_\text{tes})$ for each novel class, and form an augmented support set $\tilde{\mathcal{S}}_\text{tes}=\{Z_\text{spt}, Z_\text{gen}\}\in\mathbb R^{N\times(M+D)}$. Finally, we train a linear classifier on the augmented support set $\tilde{\mathcal{S}}_\text{tes}$ and directly apply it to the query set $\mathcal{Q}_\text{tes}$ to evaluate the final model performance.
We present the training procedure of our model in Algorithm \ref{pseudo}. The complexity analysis can be found in \textbf{Appendix} \ref{complexity}.

\begin{algorithm}[ht]
\caption{IMPRESS Training}
    \begin{algorithmic}[1]
    \REQUIRE A graph $\mathcal{G}=\{\mathcal{V}, \mathcal{E}, X, A\}$.
    \ENSURE The well-trained IMPRESS.
    \STATE \# \textit{Meta-training}
    \STATE Obtain initial hyperbolic node states using Eq.\ref{eq:hyperbolic}.
    \STATE Project the hyperbolic node states into the tangent space followed by graph convolution operations using Eq.\ref{encoder}.
    \STATE Map the node embeddings from the tangent space into the hyperbolic space using Eq.\ref{covert_hyperbolic}.
    \STATE Obtain the latent variable distribution and sample the variables using Eqs.\ref{vgae_distribution}, \ref{vgae_matrix}, and \ref{vgae_sample}.
    \STATE Perform node decoding with Eq.\ref{decoder}.
    \STATE Optimize the variational graph autoencoder with the loss in Eq.\ref{vgae_loss}.
    \STATE Obtain the pseudo prototypes with Eq.\ref{prototype}.
    \STATE Inject the prototype into the diffusion model with Eq.\ref{attention}.
    \STATE Optimize the diffusion model by minimizing the loss in Eq.\ref{dm_loss}.
    \STATE \# \textit{Meta-testing}
    \STATE Feed $\mathcal{S}_\text{tes}$ and $\mathcal{Q}_\text{tes}$ into the variational graph autoencoder, and Obtain $Z_\text{spt}$ and $Z_\text{qry}$.
    \STATE Sample noise vectors, perform denoising, and derive the generated clear node embeddings $Z_\text{gen}$.
    \STATE Train a linear classifier with augmented support set $\tilde{\mathcal{S}}_\text{tes}$.
    \STATE Evaluate the model performance with the query set $\mathcal{Q}_\text{tes}$.
    \STATE Return the well-trained IMPRESS.
    \end{algorithmic}
\label{pseudo}
\end{algorithm}
\section{Theoretical Analysis}
In this section, we provide the theoretical foundations that underpin the effectiveness of the proposed model.

\begin{theorem}
\label{thm:hyp_advantage}
Let $\mathcal{G} = (\mathcal{V}, \mathcal{E})$ be a $\delta$-hyperbolic graph with $\delta > 0$. For a hyperbolic variational autoencoder $\Phi(\cdot): \mathbb{H}^d \to \mathbb{H}^d$ and a Euclidean encoder $\Psi(\cdot): \mathbb{R}^d \to \mathbb{R}^d$, their representation capability satisfies:
\[
\inf_{v_i,v_j \in \mathcal{V}} \frac{\|\Phi(v_i) - \Phi(v_j)\|_{\mathbb{H}}}{\|\Psi(v_i) - \Psi(v_j)\|_\mathbb{E}} \leq \kappa(\delta)=O(e^{-\delta}),
\]
where $\kappa(\delta)$ is the curvature-dependent compression factor. When $\delta > \log n$, the generalization error upper bound of hyperbolic encoding is $\Omega(\sqrt{\log n})$ lower than Euclidean encoding.
\end{theorem}

Theorem \ref{thm:hyp_advantage} demonstrates that, compared to conventional graph encoders that learn node embeddings in Euclidean space, our hyperbolic embedding approach yields a tighter generalization bound in the presence of hierarchical structures in the graph.

\begin{theorem}
\label{thm:robustness}
For a meta-testing task $\mathcal{T}_{\text{tes}} \sim p(\mathcal{T})$ with the augmented support set $\tilde{\mathcal{S}}$ and query set $\mathcal{Q}_{\text{tes}}$, the classification error $\mathcal{E}_{\text{err}}$ satisfies:
\begin{equation}
\mathcal{E}_{\text{err}} \leq \frac{2c}{\delta_\mathbb{H}^2} \cdot \sigma_{\mathbb H}^2 + \epsilon_d \cdot L_f + \underbrace{\mathcal{O}\left( \sqrt{\frac{N \log (M+D)}{M+D} } \right)},
\end{equation}
where the intra-class variance $\mathbb{E}[\text{d}_{\mathbb{H}}^2(z_i, \mathcal{P}_j)] \leq \sigma_{\mathbb{H}}^2$ and inter-class distance $\text{d}_{\mathbb{H}}(\mathcal{P}_j, \mathcal{P}_{j^\prime}) \geq \delta_{\mathbb{H}}$. $\epsilon_d$ measures the Wasserstein distance between the samples distribution generated by diffusion model and the true one.
\end{theorem}

Theorem \ref{thm:robustness} indicates that when the data exhibits hierarchical structure, a smaller ratio of $\sigma^2_{\mathbb{H}}/\delta^2_{\mathbb{H}}$ contributes to reducing the first term in the generalization error bound. Moreover, the prototype-guided conditional generation helps to decrease the second term. Additionally, increasing the number of generated samples $D$ leads to a further reduction of the third term. 
The detailed proofs of Theorems \ref{thm:hyp_advantage} and \ref{thm:robustness} can be found in \textbf{Appendix} \ref{sec:theory}.

\section{Experiments}
\noindent \textbf{Datasets.}
To demonstrate the empirical effectiveness of our approach, we perform comprehensive experiments on a collection of widely adopted benchmark datasets for few-shot node classification. These include \textbf{CoraFull} \cite{bojchevski2017deep}, \textbf{Coauthor-CS} \cite{shchur2018pitfalls}, \textbf{Cora} \cite{yang2016revisiting}, \textbf{WikiCS} \cite{mernyei2020wiki}, \textbf{ML} \cite{bojchevski2017deep}, and \textbf{CiteSeer} \cite{yang2016revisiting}. Additionally, we assess the scalability of our method using the large-scale dataset \textbf{ogbn-arxiv} \cite{hu2020open}. The detailed statistics of these datasets are summarized in Table~\ref{tab:2}, and further descriptions can be found in \textbf{Appendix} \ref{dataset_description}.

\begin{table}[ht]
\centering
\caption{Statistics of the datasets.}
\resizebox{0.42\textwidth}{!}{
\begin{tabular}{ccccc}
\toprule
Dataset & \# Nodes & \# Edges & \# Features & \# Labels \\
\midrule
CoraFull        & 19,793   & 65,311     & 8,710     & 70 \\
Coauthor-CS     & 18,333   & 81,894     & 6,805     & 15 \\
Cora            & 2,708    & 5,278      & 1,433     & 7  \\
WikiCS          & 11,701   & 216,123    & 11,701    & 10 \\
ML              & 2,995    & 16,316     & 2,879     & 7  \\
CiteSeer        & 3,327    & 4,552      & 3,703     & 6  \\
ogbn-arxiv      & 169,343  & 1,166,243  & 128       & 40 \\
\bottomrule
\end{tabular}
}
\label{tab:2}
\end{table}

\noindent \textbf{Baselines.}
To comprehensively evaluate the superiority of our proposed model, we compare it against four categories of baseline methods. The first group is \textit{graph embedding methods} including 
\textbf{GCN} \cite{kipf2016semi} and \textbf{SGC} \cite{wu2019simplifying}. The second category is \textit{traditional meta-learning methods} including \textbf{ProtoNet} \cite{snell2017prototypical} and \textbf{MAML} \cite{finn2017model}. The third set is \textit{graph meta-learning methods} including \textbf{Meta-GNN} \cite{zhou2019meta}, \textbf{GPN} \cite{ding2020graph}, \textbf{G-Meta} \cite{huang2020graph}, \textbf{TENT} \cite{wang2022task}, \textbf{Meta-GPS} \cite{liu2022few}, \textbf{TLP} \cite{tan2022transductive}, \textbf{X-FNC} \cite{wang2023few}, \textbf{TEG} \cite{kim2023task}, \textbf{COSMIC} \cite{wang2023contrastive}, and \textbf{TaskNS} \cite{zhang2025unlocking}. The fourth type is \textit{unsupervised graph meta-learning methods} including \textbf{VNT} \cite{tan2023virtual} and \textbf{NaQ} \cite{jung2024unsupervised}. Detailed descriptions of these baselines are provided in \textbf{Appendix} \ref{baseline_description}.

\noindent \textbf{Implementation Details.}
During the meta-training stage, the encoder of variational graph autoencoder consists of a 2-layer graph convolutional operations (\textit{i.e.}, $t=1$) with a hidden dimension of 256. The network architecture of the diffusion model is composed of three sequential blocks, each containing a linear layer, a ReLU activation, together with a cross-attention mechanism. Moreover, the total number of diffusion steps $K$ is set to 1000. The adopted unsupervised clustering algorithm is DBSCAN \cite{ester1996density}. We optimize the model using the Adam optimizer \cite{diederik2014adam} with a learning rate of 0.001. In the meta-testing stage, we set the number of generated nodes $D$ to 50 for each novel class. For fair evaluation, the model performance is measured by computing the average accuracy over 50 randomly sampled meta-testing tasks. We report the average accuracy along with the corresponding standard deviation. For graph embedding baseline, a linear classifier is trained on the learned node representations. For other baselines, we adopt the hyperparameters from their original implementations. 

\begin{table*}[!ht]
\centering
\caption{Accuracies (\%) of different models on CoraFull and Coauther-CS. Bold: Best. Underline: Runner-up.}
\scriptsize
\resizebox{0.9\textwidth}{!}{
\begin{tabular}{c|cccc|cccc}
\toprule
\multirow{2}{*}{Model} & \multicolumn{4}{c|}{CoraFull} & \multicolumn{4}{c}{Coauther-CS} \\
\cmidrule(l){2-9}
& 5-way 3-shot & 5-way 5-shot & 10-way 3-shot & 10-way 5-shot & 2-way 3-shot & 2-way 5-shot & 5-way 3-shot & 5-way 5-shot \\
\midrule
GCN & 34.65±2.76 & 39.83±2.49 & 29.23±3.25 & 34.14±2.15 & 73.52±1.97 & 77.20±3.01 & 52.19±2.31 & 56.35±2.99 \\
SGC & 39.56±3.52 & 44.53±2.92 & 35.12±2.71 & 39.53±3.32 & 75.49±2.15 & 79.63±2.01 & 56.39±2.26 & 59.25±2.16 \\
\midrule
ProtoNet & 33.67±2.51 & 36.53±3.76 & 24.90±2.03 & 27.24±2.67 & 71.18±3.82 & 75.51±3.19 & 47.71±3.92 & 51.66±2.51 \\
MAML & 37.12±3.16 & 47.51±3.09 & 26.61±2.19 & 31.60±2.91 & 62.32±4.60 & 65.20±4.20 & 36.99±4.32 & 42.12±2.43 \\
\midrule
Meta-GNN & 52.23±2.41 & 59.12±2.36 & 47.21±3.06 & 53.32±3.15 & 85.76±2.74 & 87.86±4.79 & 75.87±3.88 & 68.59±2.59 \\
GPN & 53.24±2.33 & 60.31±2.19 & 50.93±2.30 & 56.21±2.09 & 85.60±2.15 & 88.70±2.21 & 75.88±2.75 & 81.79±3.18 \\
G-Meta & 57.52±3.91 & 62.43±3.11 & 53.92±2.91 & 58.10±3.02 & 92.14±3.90 & 93.90±3.18 & 75.72±3.59 & 74.18±3.29 \\
TENT & 64.80±4.10 & 69.24±4.49 & 51.73±4.34 & 56.00±3.53 & 89.35±4.49 & 90.90±4.24 & 78.38±5.21 & 78.56±4.42 \\
Meta-GPS & 65.19±2.35 & 69.25±2.52 & 61.23±3.11 & 64.22±2.66 & 90.16±2.72 & 92.39±1.66 & 81.39±2.35 & 83.66±1.79 \\
TLP & 66.32±2.10 & 71.36±4.49 & 51.73±4.34 & 56.00±3.53 & 90.35±4.49 & 90.90±4.22 & 82.30±2.05 & 78.56±4.42 \\
X-FNC & 69.32±3.10 & 71.26±4.19 & 49.63±4.45 & 53.00±3.93 & 90.95±4.29 & 92.03±4.14 & 82.93±2.02 & 84.36±3.49 \\
TEG & 72.14±1.06 & 76.20±1.39 & 61.03±1.13 & 65.56±1.03 & \underline{92.36±1.59} & 93.02±1.24 & 80.78±1.40 & 84.70±1.42 \\
COSMIC & 73.03±1.78 & 77.24±1.52 & \underline{65.79±1.36} & \underline{70.06±1.93} & 89.35±4.49 & 93.32±1.93 & 78.38±5.21 & 85.47±1.42 \\
TaskNS & 68.48±1.61 & 70.89±0.78 & 49.67±1.23 & 54.39±0.91 & 86.11±2.79 & 89.62±1.72 & 68.48±2.44 & 72.47±0.44 \\ \midrule
VNT & 55.19±2.59 & 70.16±1.36 & 45.76±1.72 & 57.92±2.26 & 86.12±3.20 & 89.95±3.29 & 80.16±2.55 & 82.92±2.36 \\
NaQ & \underline{75.18±1.81}  & \underline{79.63±0.78} & 63.97±1.58  & 68.35±0.60 & 92.13±0.75 & \underline{94.00±1.00} & \underline{86.30±0.97} & \underline{89.50±1.21} \\
\midrule
\textbf{IMPRESS} & \textbf{85.49±1.57} & \textbf{86.63±0.81} & \textbf{72.86±0.20} & \textbf{75.53±0.40} & \textbf{97.67±0.33} & \textbf{97.97±0.34} & \textbf{93.60±1.36} & \textbf{94.23±0.38} \\
\bottomrule
\end{tabular}
}
\label{tab:3}
\end{table*}

\begin{table*}[!ht]
\centering
\caption{Accuracies (\%) of different models on Cora, WikiCS, Cora-ML and CiteSeer. Bold: Best. Underline: Runner-up.}
\scriptsize
\resizebox{0.9\textwidth}{!}{
\begin{tabular}{c|cc|cc|cc|cc}
\toprule
\multirow{2}{*}{Model} & \multicolumn{2}{c|}{Cora} & \multicolumn{2}{c|}{WikiCS} & \multicolumn{2}{c|}{Cora-ML} & \multicolumn{2}{c}{CiteSeer} \\
\cmidrule(l){2-9}
& 2-way 3-shot & 2-way 5-shot & 2-way 3-shot & 2-way 5-shot & 2-way 3-shot & 2-way 5-shot & 2-way 3-shot & 2-way 5-shot \\
\midrule
GCN & 69.96±2.52 & 67.95±2.36 & 72.55±2.31 & 74.55±2.91 & 67.45±2.11 & 69.22±2.01 & 53.79±2.39 & 55.76±2.56 \\
SGC & 70.15±1.99 & 70.67±2.11 & 74.15±2.57 & 76.34±2.59 & 69.51±3.10 & 70.29±1.95 & 55.12±2.59 & 57.25±2.79 \\
\midrule
ProtoNet & 52.67±2.28 & 57.92±2.34 & 53.56±2.31 & 55.25±2.31 & 68.40±2.60 & 76.94±2.39 & 52.19±2.96 & 53.75±2.49 \\
MAML & 55.07±2.36 & 57.39±2.23 & 52.51±2.35 & 54.33±2.19 & 68.61±2.51 & 71.27±2.62 & 52.75±2.75 & 54.36±2.39 \\
\midrule
Meta-GNN & 70.40±2.64 & 72.51±1.91 & 78.14±2.28 & 78.35±2.60 & 70.21±2.71 & 74.34±2.41 & 59.71±2.79 & 61.32±3.22 \\
GPN & 74.05±1.96 & 76.39±2.33 & \underline{86.55±2.16} & \underline{87.80±1.95} & 80.70±2.41 & 83.21±2.15 & 64.22±2.92 & 65.59±2.49 \\
G-Meta & 74.39±2.69 & 80.05±1.98 & 61.09±2.84 & 78.35±2.60 & \underline{88.68±1.68} & 92.16±1.14 & 57.59±2.42 & 62.49±2.30 \\
TENT & 58.25±2.23 & 66.75±2.19 & 68.85±2.42 & 70.35±2.26 & 55.65±2.19 & 58.30±2.05 & 65.20±3.19 & 67.59±2.95 \\
Meta-GPS & 80.29±2.15 & 83.79±2.10 & 85.72±2.10 & 87.05±1.35 & 87.91±2.12 & 91.66±2.52 & 69.95±2.02 & 72.56±2.06 \\
TLP & 71.10±1.66 & \underline{86.15±2.19} & 83.09±2.72 & 70.35±2.26 & 85.32±2.12 & 58.30±2.05 & 71.10±2.17 & 75.55±2.03 \\
X-FNC & 78.19±3.25 & 82.70±3.19 & 83.80±3.42 & 86.30±3.20 & 85.62±3.12 & 90.36±2.99 & 67.96±3.10 & 70.29±3.05 \\
TEG & \underline{80.65±1.53} & 84.50±2.01 & \textbf{86.59±2.32} & 87.70±2.49 & 67.90±2.10 & 71.10±2.02 & \underline{73.79±1.59} & \underline{76.79±2.12} \\
COSMIC & 65.37±2.49 & 69.10±2.30 & 85.51±2.30 & 86.72±1.90 & 66.02±2.29 & 72.10±2.25 & 70.22±2.56 & 75.10±2.30 \\
TaskNS & 72.09±3.19 & 73.29±1.85 & 74.67±1.44 & 76.98±0.83 & 74.60±4.26 & 76.80±2.26 & 68.36±1.44 & 71.42±0.56 \\
\midrule
VNT & 75.32±1.96 & 79.39±2.49 & 76.25±2.42 & 85.79±2.36 & 86.22±2.45 & 89.26±1.76 & 69.25±2.16 & 73.90±2.10 \\
NaQ & 69.87±2.60 & 75.81±2.49 & 80.00±2.09 & 86.81±2.42 & 87.37±1.22 & \underline{92.56±1.41} & 72.13±2.40 & 76.41±1.55 \\
\midrule
\textbf{IMPRESS} & \textbf{89.83±0.82} & \textbf{90.80±1.36} & 83.87±1.31 & \textbf{87.93±2.07} & \textbf{97.10±0.94} & \textbf{97.63±0.61} & \textbf{75.93±2.21} & \textbf{78.03±1.13} \\
\bottomrule
\end{tabular}
}
\label{tab:4}
\end{table*}
\section{Results}

\noindent \textbf{Model Performance.}
Tables \ref{tab:3}, \ref{tab:4}, and \ref{tab:5} present the comparative performance of our proposed model and a range of baseline methods under different few-shot settings across multiple benchmark datasets. For datasets with a larger number of classes, such as CoraFull, Coauthor-CS, and ogbn-arxiv, we evaluate the performance under four different few-shot configurations. For datasets with fewer categories, namely Cora, WikiCS, Cora-ML, and CiteSeer, we consider two few-shot scenarios. The experimental results clearly demonstrate that our model consistently outperforms existing baselines on the majority of benchmarks, regardless of dataset scale. This consistent advantage provides strong evidence of the model’s effectiveness and robustness in addressing the few-shot node classification problem on graphs. We attribute this superior performance to two critical components in our model design. First, node features are mapped into hyperbolic space to better capture hierarchical structures inherent in graph data, thereby enhancing the representational capacity of node embeddings. Second, during the meta-training stage, we train a class-conditional diffusion model capable of generating novel samples. These synthetic samples are used to augment the support set during the meta-testing phase, effectively alleviating the issue of label scarcity in few-shot settings. Together, these strategies contribute to the overall advantage of our model over existing approaches.

Furthermore, we observe that graph meta-learning methods significantly outperform other types of models, which is consistent with our expectations. A key advantage lies in their ability to jointly capture both node features and structural information inherent in graphs. In addition, designed explicitly for few-shot learning on graph-structured data, these models can rapidly adapt to new tasks with only a limited number of labeled nodes. In contrast, traditional meta-learning methods such as ProtoNet and MAML exhibit consistently inferior performance across all evaluated datasets. This can be attributed to their complete ignorance of the crucial structural information within graphs. 

\begin{table}[ht]
\centering
\caption{Accuracies (\%) of different models on ogbn-arxiv. Bold: Best. Underline: Runner-up. OOM: Out-of-memory.}
\resizebox{0.42\textwidth}{!}{
\begin{tabular}{c|cccc}
\toprule
\multirow{2}{*}{Model} & \multicolumn{4}{c}{ogbn-arxiv} \\
\cmidrule(l){2-5}
& 5-way 3-shot & 5-way 5-shot & 10-way 3-shot & 10-way 5-shot \\
\midrule
GCN & 32.26±2.11 & 36.29±2.39 & 30.21±1.95 & 33.96±1.59 \\
SGC & 35.19±2.76 & 39.76±2.95 & 31.99±2.12 & 35.22±2.52 \\
\midrule
ProtoNet & 39.99±3.28 & 47.31±1.71 & 32.79±2.22 & 37.19±1.92 \\
MAML & 28.35±1.49 & 29.09±1.62 & 30.19±2.97 & 36.19±2.29 \\
\midrule
Meta-GNN & 40.14±1.94 & 45.52±1.71 & 35.19±1.72 & 39.02±1.99 \\
GPN & 42.81±2.34 & 50.50±2.13 & 37.36±1.99 & 42.16±2.19 \\
G-Meta & 40.48±1.70 & 47.16±1.73 & 35.49±2.12 & 40.95±2.70 \\
TENT & 50.26±1.73 & 61.38±1.72 & 42.19±1.16 & 46.29±1.29 \\
Meta-GPS & 52.16±2.01 & 62.55±1.95 & 42.96±2.02 & 46.86±2.10 \\
TLP & 41.96±2.29 & 52.99±2.05 & 39.42±2.15 & 42.62±2.09 \\
X-FNC & 52.36±2.75 & 63.19±2.22 & 41.92±2.72 & 46.10±2.16 \\
TEG & \underline{57.35±1.14} & 62.07±1.72 & \underline{47.41±0.63} & \underline{51.11±0.73} \\
COSMIC & 52.98±2.19 & \underline{65.42±1.69} & 43.19±2.72 & 47.59±2.19 \\
TaskNS & 46.26±2.29          & 55.99±2.05 & 42.11±1.99  & 47.25±2.15 \\
\midrule
VNT & OOM  & OOM & OOM & OOM 
\\
NaQ & OOM  & OOM & OOM & OOM \\
\midrule
\textbf{IMPRESS} & \textbf{61.11±0.45} & \textbf{67.11±1.74} & \textbf{48.72±0.38} & \textbf{53.30±1.45} \\
\bottomrule
\end{tabular}
}
\label{tab:5}
\end{table}

\noindent \textbf{Ablation Study.}
To assess the effectiveness of the proposed components, we construct two model variants and evaluate them across various few-shot learning settings on all datasets. (I) \textit{w/o hyp}: This variant removes the transformation between Euclidean and hyperbolic spaces. It only perform node representation learning in Euclidean space. (II) \textit{w/o dif}: In this setting, the diffusion model is excluded, so the support set of the meta-testing stage remains limited to only $M$ examples without any expansion. The experimental results are shown in Table~\ref{tab:7}.


\begin{table}[h]
\centering
\caption{Results of different model variants on all datasets.}
\resizebox{0.42\textwidth}{!}{
\begin{tabular}{c|c|c|c}
\toprule
Dataset (Setting) & \textit{w/o hyp} & \textit{w/o dif} & Ours \\
\midrule
CoraFull (5-way 3-shot) & 81.60±1.57 & 79.45±5.50 & \textbf{85.49±1.57} \\
Coauthor-CS (2-way 3-shot) & 97.63±0.78 & 97.13±1.40 & \textbf{97.67±0.33} \\
Cora (2-way 3-shot) & 89.80±0.86 & 84.90±6.51 & \textbf{89.83±0.82} \\
WikiCS (2-way 5-shot) & 86.03±0.69 & 87.47±0.41 & \textbf{87.93±2.07} \\
Cora-ML (2-way 5-shot) & 97.13±0.56 & 96.67±0.19 & \textbf{97.63±0.61} \\
CiteSeer (2-way 5-shot) & 76.00±2.55 & 72.97±1.90 & \textbf{78.03±1.13} \\
ogbn-arxiv (5-way 5-shot) & 64.28±1.57 & 62.77±1.17 & \textbf{67.11±1.74} \\
\bottomrule
\end{tabular}
}
\label{tab:7}
\end{table}

We have the following observations: First, the performance of both variants are inferior to our model on all datasets. This demonstrates that removing any of the proposed components leads to a significant degradation in performance, confirming the effectiveness and necessity of each individual module. Second, the inferior performance of the variant \textit{w/o hyp} highlights the limitation of Euclidean geometry in representing hierarchical structures commonly present in graph data. 
Third, the performance degradation observed in the \textit{w/o dif} variant. This phenomenon indicates that the diffusion model plays a crucial role in mitigating the challenges of few-shot learning. By generating additional class-conditional samples to expand the support set, it effectively alleviates issues such as data sparsity and distributional bias between the support and query sets. 

\noindent \textbf{Hyparameter Sensitivity.}
We conduct a systematic study on the impact of the curvature $c$ of hyperbolic space on model performance, as illustrated in the left part of Fig. \ref{fig:hyper}. For convenience, we denote the horizontal axis as $|c|$, representing the absolute value of curvature. 
The model generally achieves optimal performance when $c$ lies within the range of 0.5 to 2.0. 
These results suggest that setting $c$ too small makes the hyperbolic space close to the Euclidean space, thereby weakening the model's ability to capture hierarchical structures. 
However, if the curvature becomes too large, the space becomes overly distorted, which can lead to gradient explosion and poor convergence. 

Moreover, we study how the number of support samples $D$ generated by the diffusion model during the meta-testing phase affects the performance of our model, as shown in the right part of Fig. \ref{fig:hyper}. As the number of generated samples increases, the model performance shows a trend of first increasing and then decreasing, reaching the highest when the number is 50. This trend can be attributed to the fact that 
the diffusion model helps expand the support set with class-conditional generated samples, which alleviates the distributional shift between support and query sets. 

\begin{figure}[htbp]
  \centering
  \includegraphics[width=0.45\textwidth]{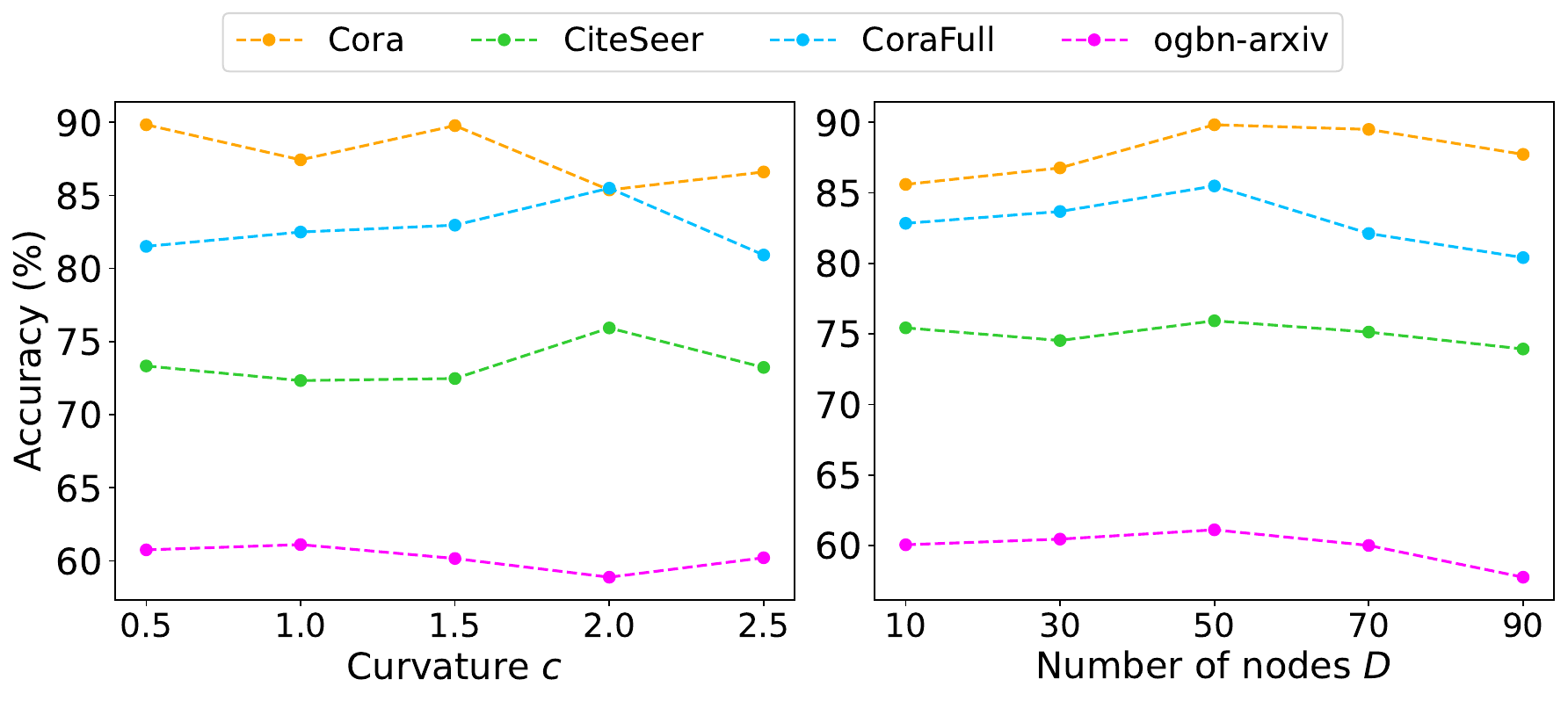}
  \caption{Performance varies with the curvature $c$ of the hyperbolic space (Left) and with the number of nodes $D$ generated by the diffusion model (Right).}
  \label{fig:hyper}
\end{figure}

We also provide empirical evidence demonstrating that our model is capable of learning hierarchical representations of nodes in \textbf{Appendix} \ref{hierarchical}. The related work is presented in \textbf{Appendix} \ref{cite_work}.
\section{Conclusion}
In this paper, we propose IMPRESS, a novel framework for graph few-shot learning. During the meta-training phase, we learn node embeddings in hyperbolic space to better capture the hierarchical structure of graph data, and simultaneously train a class-conditional denoising diffusion model. In the meta-testing phase, the trained diffusion model is used to generate additional samples, effectively expanding the limited support set. A classifier is then trained on the augmented support set and evaluated on the corresponding query set. We validate the effectiveness of our approach through both theoretical analysis and empirical experiments, demonstrating significant improvements over existing baseline methods.


\bibliography{example_paper}
\bibliographystyle{icml2026}

\newpage
\appendix
\onecolumn

\section{Appendix}

\subsection{Descriptions of Symbols}
\label{description_symbols}

We summarize the meanings of the main symbols in Table \ref{sym}.

\begin{table}[ht]
\centering
\caption{Detailed descriptions of important symbols used in our work.}
\begin{tabular}{c|c}
\toprule
\textbf{Symbols} & \textbf{Descriptions} \\
\midrule
$\mathcal{G}$, $\mathcal{V}$, $\mathcal{E}$ & graph, node set, and edge set \\
$X$, $A$ & node features, adjacency matrix \\
$\mathcal{T}_\text{tra}$, $\mathcal{T}_\text{tes}$ & meta-training and meta-testing tasks \\
$\mathcal{S}_{\text{test}}$, $\mathcal{Q}_{\text{test}}$ & support and query sets in the meta-testing task \\
$X^\mathbb H$ & hyperbolic node embeddings \\
$X^\mathcal{T}$ & tangent node embeddings \\
$\tilde{H}$, $\tilde{S}$ & refined node and set embeddings \\
$Z$ & latent node embeddings from the decoder \\
$Z_k$ & diffused node embeddings at step $k$ \\
$Z_{\text{spt}}$, $Z_{\text{qry}}$ & support and query embeddings \\
$Z_\text{gen}$ & generated node embeddings \\

\bottomrule
\end{tabular}
\label{sym}
\end{table}

\subsection{Complexity Analysis}
\label{complexity}
The overall computational cost of our framework mainly arises from the hyperbolic graph autoencoder and the diffusion model. The core operation of the hyperbolic graph autoencoder is graph convolution, whose time complexity is $O(|\mathcal{E}|d+|\mathcal{V}|d^2)$ , where $|\mathcal{E}|$ and $|\mathcal{V}|$ are the number of edges and nodes. $d$ is the feature dimension.

For the diffusion model, denoising is performed over $K$ steps for all nodes, resulting in a time complexity of $O(K|B|^2d+|B|d^2)$, where $|B|$ represents the number of nodes in each batch. Since both $|B|$ and $d$ are not particularly large in practice, the overall time complexity remains within an acceptable range.

\subsection{Theoretical Proofs}
\label{sec:theory}

\subsubsection{Proof of Theorem \ref{thm:hyp_advantage}}
\begin{proof}
The proof consists of three key steps:

\noindent \textbf{Step 1: Hyperbolic Distance Preservation.}

For any $v_i, v_j \in \mathcal{V}$, the $\delta$-hyperbolicity implies that the Gromov product $(v_i \cdot v_j)_o \geq \frac{1}{2}(d_g(o,v_i) + d_g(o,v_j) - d_g(v_i,v_j)) - \delta$ for any base point $o$. In the Poincaré ball model $\mathbb{H}^d$, this translates to:
\[
d_{\mathbb{H}}(\Theta(v_i), \Theta(v_j)) \leq 2\delta + 2\log_2(d_g(v_i,v_j)) + C_1,
\]
where $C_1$ accounts for the embedding distortion. The exponential map $\exp_o(v) = \tanh(\|v\|/2)\frac{v}{\|v\|}$ ensures that graph distances are preserved up to $\delta$-dependent distortion.

\noindent \textbf{Step 2: Euclidean vs. Hyperbolic Distortion.}

By Bourgain's theorem \cite{bourgain1985lipschitz}, any Euclidean embedding of a metric space must incur at least $\Omega(\sqrt{\log n})$ distortion for some node pairs. For hyperbolic embeddings, the distortion is bounded differently. Using the hyperbolic law of cosines:
\[
\cosh(\zeta d_{\mathbb{H}}(x,y)) = \cosh(\zeta r_x)\cosh(\zeta r_y) - \sinh(\zeta r_x)\sinh(\zeta r_y)\cos(\theta_{xy}).
\]
For $\delta$-hyperbolic graphs, setting $\zeta = \delta^{-1}$ yields:
\[
d_{\mathbb{H}}(\Theta(v_i), \Theta(v_j)) \leq (1 + \epsilon)\cdot d_g(v_i,v_j) + O(\delta),
\]
with $\epsilon = O(e^{-\delta})$. The ratio $\frac{d_{\mathbb{H}}}{d_{\mathbb{E}}} \leq \kappa(\delta)$ follows.

\noindent \textbf{Step 3: Generalization Error Bound.}

The variational autoencoder's evidence lower bound objective induces a latent space geometry where:
\[
\mathscr{D}_{\text{KL}}^{\mathbb{H}}(q_\theta(z|x) \| p(z)) \leq \frac{C_2 \text{dim}(z)}{\delta} \log(1/\epsilon),
\]
compared to the Euclidean KL divergence $\mathscr{D}_{\text{KL}}^{\mathbb{E}} \geq \Omega(\text{dim}(z)\|\mu\|^2)$. For a $\delta$-hyperbolic graph with $\delta > \log n$, the hyperbolic encoder's generalization error $R_{\mathbb{H}}$ satisfies:
\[
R_{\mathbb{H}} \leq R_{\mathbb{E}} - C_3 \sqrt{\log n},
\]
\end{proof}



\subsubsection{Proofs of Theorem \ref{thm:robustness}}

The classification error bound in Theorem \ref{thm:robustness} consists of three components. In the following, we provide a detailed analysis and proof for each of them individually.

\begin{proof}[Proof of the First Term]
In the Poincaré ball model, node embeddings $z_i$ follow $\mathcal{N}_{\mathbb{H}}(\mathcal{P}_j, \sigma_{\mathbb{H}}^2)$. For any $z_i \in \tilde{\mathcal{S}}$, the hyperbolic variance assumption implies:

\begin{equation}
\mathbb{E}[\cosh^2(\zeta \text{d}_{\mathbb{H}}(z_i, \mathcal{P}_j)) ] \leq 1 + \frac{2}{\zeta^2} \sigma_{\mathbb{H}}^2, \quad \zeta = \sqrt{-c}.
\end{equation}

Construct the separation function using McShane extension:
\begin{equation}
h_j(z_i) = \frac{ \cosh(\zeta \delta_{\mathbb{H}}) - \cosh(\zeta \text{d}_{\mathbb{H}}(z_i, \mathcal{P}_j)) }{ \zeta^2 \sinh(\zeta \delta_{\mathbb{H}}) }.
\end{equation}

This satisfies:
\begin{equation}
|h_j(z_i) - \mathbb{I}(y_i = j)| \leq \frac{2 \cosh(\zeta \text{d}_{\mathbb{H}}(z_i, \mathcal{P}_j))}{\zeta^2 \sinh(\zeta \delta_{\mathbb{H}})}.
\end{equation}

Combining with variance bound:
\begin{align}
&\mathbb{E} \left| h_j(z_i) - \mathbb{I}(y_i = j) \right| \leq \frac{2(1 + \frac{2}{\zeta^2} \sigma_{\mathbb{H}}^2)}{\zeta^2 \sinh(\zeta \delta_{\mathbb{H}})} \\
&\leq \frac{2c}{\delta_{\mathbb{H}}^2} \sigma_{\mathbb{H}}^2 \quad \text{(using $\sinh(z) \approx z$ for moderate $z$)}.
\end{align}
\end{proof}

\begin{proof}[Proof of the Second Term]
For the Lipschitz classifier $f$ with $\mathcal{W}$ distance assumption:
\begin{align}
|\hat{\mathcal{R}}_{\text{dm}} - \mathcal{R}_{\text{true}}| &\leq L_f \cdot \mathbb{E}_{z \sim q_{\text{dm}}} [ \text{d}_{\mathcal{N}}(z, z_{\text{true}}) ] \\
&\leq L_f \cdot \epsilon_d
\end{align}
where $C$ is the Lipschitz constant for Poincaré-to-tangent space projection.
\end{proof}

\begin{proof}[Proof of the third term]
To prove the third term, we first present the following lemma.
\begin{lemma}
\label{lem:variance}
For a $N$-class classification problem with $M+D$ generated samples per class, let $\mathcal{F}$ be the hypothesis class with Rademacher complexity $\mathcal{R}(\mathcal{F})$. The variance term is bounded by:
\begin{equation}
\sup_{f \in \mathcal{F}} \left| \hat{\mathcal{R}}_{\text{dm}}(f) - \mathbb{E}[\hat{\mathcal{R}}_{\text{dm}}(f)] \right| \leq \mathcal{O}\left( \sqrt{\frac{N \log (M+D)}{M+D}} \right)
\end{equation}
with probability at least $1 - \wp$ for $\wp \in (0,1)$.
\end{lemma}

\begin{proof}
We proceed in four steps:

\noindent \textbf{Step 1: McDiarmid's Inequality}\\
Define $\varphi(\mathbf{z}_1^{(1)},...,\mathbf{z}_{M+D}^{(N)}) = \sup_{f \in \mathcal{F}} |\hat{\mathcal{R}}_{\text{dm}}(f) - \mathbb{E}[\hat{\mathcal{R}}_{\text{dm}}(f)]|$. If we change one sample $\mathbf{z}_i^{(k)}$ to $\tilde{\mathbf{z}}_i^{(k)}$, the variation is bounded by:
\begin{equation}
|\varphi(\cdot) - \varphi(z_i^{(K)} \to \tilde{z}_i^{(K)})| \leq \frac{L_\ell}{N(M+D)}
\end{equation}
where $L_\ell$ is the Lipschitz constant of the loss function. McDiarmid's inequality gives:
\begin{equation}
\mathbb{P}\left( \varphi - \mathbb{E}[\varphi] \geq \tau \right) \leq \exp\left( -\frac{2N(M+D) \tau^2}{L_\ell^2} \right)
\end{equation}

\noindent \textbf{Step 2: Symmetrization}\\
The expected deviation $\mathbb{E}[\varphi]$ can be bounded via Rademacher complexity:
\begin{equation}
\mathbb{E}[\varphi] \leq 2\mathcal{R}_M(\mathcal{F}) + L_\ell\sqrt{\frac{\log(1/\wp)}{2N(M+D)}}
\end{equation}
where $\mathcal{R}_M(\mathcal{F}) = \mathbb{E}\left[ \sup_{f \in \mathcal{F}} \frac{1}{N(M+D)} \sum_{n,i} \varsigma_i^{(n)} \ell(f(z_i^{(n)}),n) \right]$ with Rademacher variables $\varsigma_i^{(k)} \in \{\pm1\}$.

\noindent \textbf{Step 3: Class-Wise Complexity}\\
For $N$ classes, the Rademacher complexity decomposes as:
\begin{align}
\mathcal{R}_M(\mathcal{F}) &\leq \sum_{n=1}^N \mathbb{E}\left[ \sup_{f \in \mathcal{F}} \frac{1}{N(M+D)} \sum_{i=1}^{M+D} \varsigma_i^{(n)} \ell(f(z_i^{(n)}),n) \right] \\
&\leq \sqrt{\frac{2N \log|\mathcal{F}|}{M}}
\end{align}

\noindent \textbf{Step 4: Union Bound}\\
Setting $\wp = 1/M+D$ and combining Steps 1-3:
\begin{equation}
\varphi \leq \underbrace{2\sqrt{\frac{2N \log|\mathcal{F}|}{M+D}}}_{\text{Rademacher}} + \underbrace{L_\ell\sqrt{\frac{\log (M+D)}{K(M+D)}}}_{\text{McDiarmid}} = \mathcal{O}\left( \sqrt{\frac{N \log (M+D)}{(M+D)}} \right)
\end{equation}
\end{proof}

Combing the above terms, we have:
\begin{equation}
\mathcal{E}_{\text{err}} \leq \underbrace{\frac{2B\sigma_{\mathbb{H}}^2}{\delta_{\mathbb{H}}^2}}_{\text{Hyperbolic}} + \underbrace{\epsilon_d L_f}_{\text{Diffusion}} + \underbrace{\sqrt{\frac{N\log (M+D)}{M+D}}}_{\text{Variance}}
\end{equation}
Thus, we complete the proof of Theorem \ref{thm:robustness}.
\end{proof}

\subsection{Details of Datasets}
\label{dataset_description}
In this section, we provide the detailed descriptions of these evaluated datasets.

\noindent \textbf{CoraFull} \cite{bojchevski2017deep}: It is a well-known citation network, where each node corresponds to a scientific publication and edges denote citation links between papers. The nodes are categorized according to the topics of the papers. In our experiments, we divide the dataset into 25, 20, and 25 classes for meta-training, meta-validation, and meta-testing, respectively.

\noindent \textbf{Coauthor-CS} \cite{shchur2018pitfalls}: It represents a co-authorship graph, in which nodes correspond to individual researchers, and an edge exists if two researchers have co-authored a paper. For this dataset, we follow a 5/5/5 class split for the meta-training, meta-validation, and meta-testing phases.

\noindent \textbf{Cora} \cite{yang2016revisiting}: Similar to CoraFull, it is also a citation network in which nodes denote papers and edges indicate citation relationships. We adopt a 3/2/2 class division for meta-training, meta-validation, and meta-testing tasks.

\noindent \textbf{WikiCS} \cite{mernyei2020wiki}: It is derived from Wikipedia articles in the field of computer science. Each node is an article, and edges are formed based on the hyperlink. The dataset is partitioned into 4, 3, and 3 classes for meta-training, meta-validation, and meta-testing.

\noindent \textbf{ML} \cite{bojchevski2017deep}: It focuses exclusively on machine learning papers, treating each node as a paper and linking them via citation relationships. For this dataset, we apply a 3/2/2 class split.

\noindent \textbf{CiteSeer} \cite{yang2016revisiting}: It is another citation-based graph, where scientific articles are nodes and citation links serve as edges. We utilize 2 classes each for training, validation, and testing in the meta-learning framework.

\noindent \textbf{ogbn-arxiv} \cite{hu2020open}: It is a large-scale citation network of arXiv papers in computer science. The nodes represent individual papers, and the edges are constructed based on citation relationships. Each node is labeled according to one of 40 subject areas defined in arXiv. We use a class split of 20/10/10 for this dataset.

\subsection{Details of Baselines}
\label{baseline_description}
In this section, we present the detailed descriptions of the used baselines, which includes four categories: \textit{Graph embedding methods}, \textit{Traditional meta-learning methods}, \textit{Graph meta-learning methods}, and \textit{Unsupervised graph meta-learning methods}.

\vspace{0.5em}
\noindent \textit{A.4.1 \quad Graph Embedding Methods.}

\noindent \textbf{GCN} \cite{kipf2016semi}:
It employs a first-order approximation of Chebyshev graph filters to efficiently propagate information across nodes, enabling the extraction of meaningful hidden representations for downstream tasks.

\noindent \textbf{SGC} \cite{wu2019simplifying}:
It mitigates the redundant complexity of GCN by eliminating the non-linear activations and collapsing the weight matrices.

\vspace{0.5em}
\noindent \textit{A.4.2 \quad Traditional Meta-learning Methods.}

\noindent \textbf{ProtoNet} \cite{snell2017prototypical}:
It addresses few-shot classification by learning a metric space where samples are classified based on their proximity to class-specific prototypes.

\noindent \textbf{MAML} \cite{finn2017model}:
By performing a few gradient updates during meta-training phase, it acquires a parameter initialization that facilitates efficient adaptation to novel tasks with limited labeled data.

\vspace{0.5em}
\noindent \textit{A.4.3 \quad Graph Meta-learning Methods.}

\noindent \textbf{Meta-GNN} \cite{zhou2019meta}:
It integrates MAML with GNN, learning transferable knowledge from few-shot tasks, allowing the model to quickly adapt to new node classification tasks with limited labeled nodes.

\noindent \textbf{TENT} \cite{wang2022task}: This framework introduces three adaptation modules operating at node-level, class-level, and task-level. These components bridge the generalization gap between meta-training and meta-testing, and enable the model to flexibly adjust to diverse meta-tasks.

\noindent \textbf{Meta-GPS} \cite{liu2022few}
It is a graph meta-learning framework that combines prototype-based initialization and scaling and shifting transformations to enable effective transferable knowledge learning and adaptation.

\noindent \textbf{TLP} \cite{tan2022transductive}:
It employs transductive linear probing by first pretraining a graph encoder via contrastive learning and subsequently utilizing it to generate node embeddings for meta-testing.

\noindent \textbf{X-FNC} \cite{wang2023few}:
It applies Poisson-based label propagation to generate abundant pseudo-labeled nodes, followed by node classification and an information bottleneck strategy to filter out irrelevant signals.

\noindent \textbf{TEG} \cite{kim2023task}:
It proposes a task-equivariant framework for graph few-shot learning that leverages equivariant neural networks to capture task-specific inductive biases and learn adaptive strategies for individual tasks.

\noindent \textbf{COSMIC} \cite{wang2023contrastive}:
It presents a contrastive meta-learning framework for graphs, which aligns node embeddings within each class through a two-step contrastive optimization and improves cross-class generalization by generating challenging node classes via a similarity-sensitive mix-up strategy.

\noindent \textbf{TaskNS} \cite{zhang2025unlocking}:
It incorporates task-specific negative samples into meta-training, employs a topology-aware sampling scheme, and adopts a customized loss function to enhance intra-class compactness and inter-class separation.

\vspace{0.5em}
\noindent \textit{A.4.4 \quad Unsupervised Graph Meta-learning Methods.}

\noindent \textbf{VNT} \cite{tan2023virtual}:
It employs a pretrained graph transformer encoder and optimizes virtual nodes as soft prompts using few-shot labels to adapt node embeddings for each task.

\noindent \textbf{NaQ} \cite{jung2024unsupervised}:
It introduces an unsupervised episode generation strategy, leveraging neighbor-based query sampling to fully exploit unlabeled graph data while maintaining compatibility with existing supervised graph meta-learning models.

\begin{table*}[ht]
\centering
\caption{Results of hierarchical clustering for different models.}
\label{cluster}
\begin{tabular}{@{}c|cccccccccccc@{}}
\toprule
\multirow{2}{*}{Model} & \multicolumn{2}{c}{CoraFull}      & \multicolumn{2}{c}{Coauthor-CS}       & \multicolumn{2}{c}{Cora}      & \multicolumn{2}{c}{WikiCS}      & \multicolumn{2}{c}{ML}        & \multicolumn{2}{c}{CiteSeer}         \\ \cmidrule(l){2-13} 
                       & SC$\uparrow$   & DB$\downarrow$ & SC$\uparrow$    & DB$\downarrow$ & SC$\uparrow$   & DB$\downarrow$ & SC$\uparrow$   & DB$\downarrow$ & SC$\uparrow$   & DB$\downarrow$ & SC$\uparrow$   & DB$\downarrow$ \\ \midrule
TEG                  & 0.099          & 1.591          &    0.203          & 1.304            & 0.159          & 1.835          & 0.068          & 1.193          & 0.166         & 1.436          & 0.061          & 3.763                    \\
COSMIC                  & -0.050         & 3.807          & 0.131             & 1.934            & 0.068          & 3.272          & 0.036            & 2.493            & 0.033          & 3.312          & -0.007          & 4.926                    \\
Ours                   & \textbf{0.105} & \textbf{1.560} & \textbf{0.299} & \textbf{1.217} & \textbf{0.172} & \textbf{1.728} & \textbf{0.199} & \textbf{1.100} & \textbf{0.211} & \textbf{1.359} & \textbf{0.113} & \textbf{1.966} \\ \bottomrule
\end{tabular}%
\end{table*}

\subsection{Hierarchical Property}
\label{hierarchical}
To demonstrate that our model effectively captures the hierarchical structure of nodes through hyperbolic embeddings, we perform hierarchical clustering on the learned node representations and compare the results against two representative baselines, TEG and COSMIC. We adopt the Silhouette Coefficient (SC) score (higher is better) and the Davies-Bouldin (DB) score (lower is better) as evaluation metrics, and the results are summarized in Table \ref{cluster}. As can be clearly observed in Table \ref{cluster}, our model consistently achieves the best performance in hierarchical clustering, indicating that the node embeddings learned in hyperbolic space exhibit more pronounced hierarchical structures compared to the baselines trained in Euclidean space.

\subsection{Related Work}
\label{cite_work}
\subsubsection{Graph Few-shot Learning}
Graph FSL aims to equip models with the ability to rapidly generalize to novel tasks with scarce labeled data by utilizing meta-knowledge obtained from prior learning experiences. This paradigm has been increasingly adopted in various real-world applications \cite{tan2023virtual, wang2023contrastive, zhang2025few}. Current graph FSL methods can be broadly classified into two categories: optimization-based \cite{huang2020graph, kim2023task, liu2022few, zhou2019meta} and metric-based \cite{ding2020graph, liu2024simple, yao2020graph}. Optimization-based methods typically integrate graph neural networks (GNNs) with meta-learning algorithms such as MAML \cite{finn2017model}, enabling models to adapt their parameters efficiently across tasks. Metric-based methods focus on learning a task-general metric space that can effectively measure distance among numerous unlabeled nodes and a small number of labeled ones. Most existing graph FSL methods acquire transferable knowledge during the meta-training phase, which is constructed from a large set of base-class tasks with sufficient labeled data \cite{zhou2019meta, liu2024simple}. Representative examples include Meta-GPS \cite{liu2022few}, TEG \cite{kim2023task}, and TaskNS \cite{zhang2025unlocking}. Meta-GPS facilitates knowledge generalization by combining parameter initialization with task-wise scaling and shifting transformations. TEG attempts to model task-level inductive biases by adopting equivariant neural networks, thus alleviating the dependency on labeled supervision. TaskNS aims to learn discriminative node embeddings by performing negative sampling in each task.

\subsubsection{Hyperbolic Neural Networks}
In recent years, hyperbolic embeddings have emerged as a powerful tool in representation learning, particularly for capturing hierarchical and tree-like structures \cite{mettes2024hyperbolic}. Poincaré embeddings \cite{nickel2017poincare} are introduced for learning symbolic representations within hyperbolic geometry and have shown strong capabilities in modeling hierarchical relationships. Following this, an alternative optimization approach based on the Lorentz model of hyperbolic space is proposed \cite{nickel2018learning}, which substantially improved embedding quality. 
Recent studies have attempted to incorporate hyperbolic space into graph-structured data.
For instance, HGCN \cite{chami2019hyperbolic} integrates the representational capacity of GCNs with hyperbolic geometry to learn inductive node representations for hierarchical and scale-free graphs. The concurrently developed HGNN \cite{liu2019hyperbolic} is also proposed for learning GNNs in hyperbolic space. 
Inspired by these developments, our model leverages the Poincaré ball model to map node embeddings from Euclidean space into hyperbolic space, aiming to better preserve hierarchical structures in the latent space.

\subsubsection{Diffusion Models}
Diffusion models are a class of generative approaches that learn data distributions and generate new samples through the process of gradually adding noise and inverse denoising \cite{croitoru2023diffusion}. 
The well-known denoising diffusion probabilistic models \cite{ho2020denoising} lay the foundation for modern diffusion models by introducing a simplified framework consisting of forward process and reverse process. 
Further, latent diffusion models \cite{rombach2022high} improves diffusion models by performing the diffusion process in compressed latent space rather than pixel space, which significantly improves computational efficiency while preserving generation quality. In recent years, diffusion models have been widely used in fields such as text-to-image generation \cite{rombach2022high, ramesh2022hierarchical} and text-to-audio \cite{yang2023diffsound, huang2022prodiff}. 
Motivated by their strong generative capability and flexibility, we incorporate diffusion models into our framework to enhance the performance of graph FSL.


\end{document}